\documentclass{aamas2017}

\usepackage{times}
\usepackage{helvet}
\usepackage{courier}

\usepackage[utf8]{inputenc} 
\usepackage{url}            
\usepackage{booktabs}       
\usepackage{amsfonts}       
\usepackage{nicefrac}       
\usepackage{microtype}      
\usepackage{floatrow}
\usepackage{authblk}

\usepackage{wrapfig}
\usepackage{graphicx} 
\usepackage{fixltx2e}
\usepackage{url}
\usepackage{amsfonts}

\usepackage{amsmath}
\usepackage{algorithm}
\usepackage{algorithmic}
\usepackage {tikz}
\usetikzlibrary {positioning}
\usepackage[font=footnotesize,labelfont=bf]{caption} 
\usepackage{subcaption}
\usepackage[toc,page]{appendix}

\newtheorem{theorem}{Theorem}
\newtheorem{corollary}{Corollary}

\newtheorem{lemma}{Lemma}

\newcommand\iid{i.i.d. }

\usepackage{fancyhdr}

\title{Bootstrapping with Models: Confidence Intervals for Off-Policy Evaluation}

\begin{document}
\thispagestyle{fancy}

\author[1]{Josiah P. Hanna}
\author[1]{Peter Stone}
\author[1]{Scott Niekum}

\affil[1]{Department of Computer Science, The University of Texas at Austin, Austin, Texas,  U.S.A}

\maketitle

\begin{abstract}

For an autonomous agent, executing a poor policy may be costly or even dangerous. 
For such agents, it is desirable to determine confidence interval lower bounds on the performance of any given policy \emph{without executing said policy}. 
Current methods for exact high confidence off-policy evaluation that use importance sampling require a substantial amount of data to achieve a tight lower bound.
Existing model-based methods only address the problem in discrete state spaces.
Since exact bounds are intractable for many domains we trade off strict guarantees of safety for more data-efficient approximate bounds.
In this context, we propose two bootstrapping off-policy evaluation methods which use learned MDP transition models in order to estimate lower confidence bounds on policy performance with limited data in both continuous and discrete state spaces.
Since direct use of a model may introduce bias, we derive a theoretical upper bound on model bias for when the model transition function is estimated with i.i.d.~trajectories. 
This bound broadens our understanding of the conditions under which model-based methods have high bias.
Finally, we empirically evaluate our proposed methods and analyze the settings in which different bootstrapping off-policy confidence interval methods succeed and fail.
\end{abstract}

\printccsdesc



\section{Introduction}


As \textit{reinforcement learning} (RL) agents find application in the real world, it will be critical to establish the performance of policies with high confidence before they are executed.
For example, deploying a poorly performing policy on a manufacturing robot may slow production or, in the worst case, damage the robot or harm humans working around it.
It is insufficient to have a policy that has a high off-policy predicted value --- we want to specify a lower bound on the policy's value that is correct with a pre-determined level of confidence.
This problem is known as the \textit{high confidence off-policy evaluation problem}.
We propose data-efficient approximate solutions to this problem.

High confidence off-policy model-based methods require large amounts of data and are limited to discrete settings \cite{chow2015robust}.
In continuous settings, current methods for high confidence off-policy evaluation rely on importance sampling \cite{precup2000eligibility} with existing domain data \cite{thomas2015off-policy}.
Due to the large variance of importance sampled returns, these algorithms can require prohibitively large amounts of data to produce meaningful confidence bounds. 
The current state-of-the-art for high confidence off-policy evaluation in discrete and continuous settings is a concentration inequality tailored to the distribution of importance sampled returns \cite{thomas2015off-policy}.
Unfortunately, the amount of data required for tight confidence bounds preclude the use of this method in data-scarce settings such as robotics.

Instead of exact high confidence, Thomas et al.\ \shortcite{thomas2015policy} demonstrated that approximate bounds obtained by bootstrapping importance sampled policy returns can improve data-efficiency by an order of magnitude over concentration inequalities.
In this work, we propose two approximate high confidence off-policy evaluation methods through the combination of bootstrapping with learned models of the environment's transition dynamics.
Both methods are straightforward to implement (though seemingly novel) and are empirically demonstrated to outperform importance-sampling methods.

\begin{sloppypar}
Our first contribution, \textit{Model-based Bootstrapping} (\textsc{mb-bootstrap}), directly combines bootstrapping with learned models of the environment's dynamics for off-policy value estimation.
Since \textsc{mb-bootstrap} uses direct model-based estimates of policy value, it may exhibit bias in some settings.
To characterize these settings, we derive an upper bound on model bias for models estimated from arbitrary distributions of trajectories.
Our second algorithmic contribution, \textit{weighted doubly robust bootstrapping} (\textsc{wdr-bootstrap}), combines bootstrapping with the recently proposed weighted doubly robust estimator \cite{thomas2016data-efficient} which uses a model to lower the variance of importance sampling estimators without adding model bias to the estimate.
We empirically evaluate both methods on two high confidence off-policy evaluation tasks. Our results show these methods are far more data-efficient than existing importance sampling based approaches.
Finally, we combine theoretical and empirical results to make specific recommendations about when to use different off-policy confidence bound methods in practice.
\end{sloppypar}

\section{Preliminaries}

\subsection{Markov Decision Processes}
We formalize our problem as a \textit{Markov decision process} (MDP) defined as $( \mathcal{S},\mathcal{A},P,r,\gamma, d_0)$ where
 $\mathcal{S}$ is a set of states, $\mathcal{A}$ is a set of actions, $P: \mathcal{S} \times \mathcal{A} \times \mathcal{S} \rightarrow [0.1]$ is a probability mass function defining a distribution over next states for each state and action, $r: \mathcal{S} \times \mathcal{A} \rightarrow [0,r_\mathtt{max}]$ is a bounded, non-negative reward function, $\gamma \in [0,1]$ is a discount factor, and $d_0$ is a probability mass function over initial states. 
An agent samples actions from a policy, $\pi: \mathcal{S} \times \mathcal{A} \rightarrow [0,1]$, which is a probability mass function on $\mathcal{A}$ conditioned on the current state.
A policy is deterministic if $\pi(a|s)=1$ for only one $a$ in each $s$.\footnote{We define notation for discrete MDPs, however, all results hold for continuous $\mathcal{S}$ and $\mathcal{A}$ by replacing summations with integrals and probability mass functions with probability density functions.}

\begin{sloppypar}
A trajectory, $H$ of length $L$ is a state-action history, $S_0,A_0,\dotsc,S_{L-1},A_{L-1}$ where $S_0 \sim d_0$, $A_{t} \sim \pi(\cdot|s_t)$, and $S_{t+1} \sim P(\cdot|S_t,A_t)$.
The return of a trajectory is $g(H) = \sum_{t=0}^{L-1} \gamma^{t} r(S_t,A_t)$.
The policy, $\pi$, and transition dynamics, $P$, induce a distribution over trajectories, $p_\pi$.
We write $H \sim \pi$ to denote a trajectory sampled by executing $\pi$ (i.e., sampled from $p_\pi$).
The expected discounted return of a policy, $\pi$, is defined as $V(\pi):= E_{H \sim \pi}[g(H)]$.
\end{sloppypar}


Given a set of $n$ trajectories, $\mathcal{D} = \{H_1,..,H_n\}$, where $H_i \sim \pi_b$ for some behavior policy, $\pi_b$, an evaluation policy, $\pi_e$, and a confidence level, $\delta \in [0,1]$, we propose two methods to approximate a confidence lower bound, $V_{\delta}(\pi_e)$, on $V(\pi_e)$ such that $V_{\delta}(\pi_e) \leq V(\pi_e)$ with probability at least $1-\delta$.

\subsection{Importance Sampling}

We define an \textit{off-policy estimator} as any method for computing an estimate, $\widehat{V}(\pi_e)$, of  $V(\pi_e)$ using trajectories from a second policy, $\pi_b$.
\textit{Importance sampling} (\textsc{is}) is one such method \cite{precup2000eligibility}.
For a trajectory $H \sim \pi_b$, where $H = S_1,A_1,\dotsc,S_L,A_L$, we define the importance weight up to time $t$ for policy $\pi_e$ as $ \rho_t^H:=\prod_{i=0}^t \frac{\pi_e(A_i|S_i)}{\pi_b(A_i|S_i)}$.
Then the \textsc{is} estimator of $V(\pi_e)$ with a trajectory, $H \sim \pi_b$ is defined as 
$\text{\textsc{is}}(\pi_e, H, \pi_b):= g(H) \rho_L^H$.
A lower variance version of importance sampling for off-policy evaluation is \textit{per-decision importance sampling}, 
$\text{\textsc{pdis}}(\pi_e, H, \pi_b):= \sum_{t=0}^{L-1} r(S_t,A_t) \rho_t^H$.
We overload \textsc{is} notation to define the batch \textsc{is} estimator for a set of $n$ trajectories, $\mathcal{D}$, so that \textsc{is}$(\mathcal{D}):=\frac{1}{n}\sum_{i=1}^n \text{\textsc{is}}(\pi_e,H_i,\pi_b)$.
The batch \textsc{pdis} estimator is defined similarly.

The variance of batch \textsc{is} estimators can be reduced with \textit{weighted importance sampling} (\textsc{wis}) and \textit{per-decision weighted importance sampling} (\textsc{pdwis}). Define the weighted importance weight up to time $t$ for the $i$\textsuperscript{th} trajectory as $w_t^{H_i}:=\rho_t^{H_i}/\sum_{j=1}^n \rho_t^{H_j}$. Then the \textsc{wis} estimator is defined as: \textsc{wis}$(\mathcal{D}):=\sum_{i=1}^n g(H_i) w_L^{H_i}$. \textsc{pdwis} is defined as \textsc{pdis} with $w_t^{H_i}$ replacing $\rho_t^{H_i}$.

Provided the support of $\pi_e$ is a subset of the support of $\pi_b$, \textsc{is} and \textsc{pdis} are unbiased but potentially high variance estimators.
\textsc{wis} and \textsc{pdwis} have less variance than their unweighted counterparts but introduce bias.
When all $H_i \in \mathcal{D}$ are sampled from the same policy, $\pi_b$, \textsc{wis} and \textsc{pdwis} introduce a particular form of bias. 
Namely, when $n=1$, $\widehat{V}(\pi_e)$ is an unbiased estimate of $V(\pi_b)$. 
As $n$ increases, the estimate shifts from $V(\pi_b)$ towards $V(\pi_e)$.
Thus, \textsc{wis} and \textsc{pdwis} are statistically consistent (i.e., \textsc{wis}$(\mathcal{D})\rightarrow V(\pi_e)$ as $n \rightarrow \infty$) \cite{thomas2015safe}.

\subsection{Bootstrapping}\label{sec:bootstrap}

This section gives an overview of bootstrapping \cite{efron1987better}.
In the next section we propose bootstrapping with learned models to estimate confidence intervals for off-policy estimates.

Consider a sample $X$ of $n$ random variables $X_i$ for $i=1,\dotsc,n$, where we sample $X_i$  i.i.d.~from some distribution $f$.
From the sample, $X$, we can compute a sample estimate, $\hat{\theta}$ of a parameter, $\theta$ such that $\hat{\theta} = t(X)$.
For example, if $\theta$ is the population mean, then $t(X):=\frac{1}{n}\sum_{i=1}^n X_i$.
For a finite sample, we would like to specify the accuracy of $\hat{\theta}$ without placing restrictive assumptions on the sampling distribution of $\hat{\theta}$ (e.g., assuming $\hat{\theta}$ is distributed normally).
Bootstrapping allows us to estimate the distribution of $\hat{\theta}$ from whence confidence intervals can be derived.
Starting from a sample $X=\{X_1,\dotsc,X_n\}$, we create $B$ new samples, $\tilde{X}^j = \{\tilde{X}^j_1,\dotsc,\tilde{X}^j_n\}$, by sampling $\tilde{X}^j_i$ from a bootstrap distribution, $\hat{f}$.
That is, we sample $\hat{f}$ by independently sampling $X_i$ from $X$ \textit{with replacement}.
For each $\tilde{X}^j$ we compute $\hat{\theta}_j = t(\tilde{X}^j)$.
The distribution of the $\hat{\theta}_j$ approximates the distribution of $\hat{\theta}$ which allows us to compute sample confidence bounds.
See the work of Efron \shortcite{efron1979bootstrap} for a more detailed introduction to bootstrapping.


While bootstrapping has strong guarantees as $n\rightarrow \infty$, bootstrap confidence intervals lack finite sample guarantees.
Using bootstrapping requires the assumption that the bootstrap distribution is representative of the distribution of the statistic of interest which may be false for a finite sample.
Therefore, we characterize bootstrap confidence intervals as ``semi-safe" due to this possibly false assumption.
In contrast to lower bounds from concentration inequalities, bootstrapped lower bounds can be thought of as approximating the allowable $\delta$ error rate instead of upper bounding it.
However, bootstrapping is considered safe enough for high risk medical predictions and in practice has a well established record of producing accurate confidence intervals \cite{chambless2003coronary}.
In the context of policy evaluation, Thomas et al.\ \shortcite{thomas2015policy} established that bootstrap confidence intervals with \textsc{wis} can provide accurate lower bounds in the high confidence off-policy evaluation setting.
The primary contribution of our work is to incorporate off-policy estimators of $V$ that use models into bootstrapping to decrease the data requirements needed to produce a tight lower bound.

\section{Off-Policy Bootstrapped Lower Bounds}

In this section, we propose model-based and weighted doubly robust bootstrapping for estimating confidence intervals on off-policy estimates.
First, we present pseudocode for computing a bootstrap confidence lower bound for any off-policy estimator (Algorithm \ref{alg:bootstrap}).
Our proposed methods are instantiations of this general algorithm.

\begin{figure}

\centering
\begin{minipage}[t][4.8cm][t]{\textwidth}
\small{
\begin{algorithm}[H]

{\fontsize{9}{9}\selectfont
\begin{algorithmic}[1]
\INPUT $\pi_e$, $\mathcal{D}$, $\pi_b$, $\delta$, $B$
\OUTPUT $1-\delta$ confidence lower bound on $V(\pi_e)$.

\FORALL{$i \in [1,B]$}
\STATE $\mathcal{\tilde{D}}_i \gets \{H_1^i,\dotsc,H_n^i\}$ where $H_j^i \sim \mathcal{U}({\mathcal{D}})$ // where $\mathcal{U}$ is the uniform distribution
\STATE $\widehat{V}_i \gets $ \textbf{Off-PolicyEstimate($\pi_e,\mathcal{\tilde{D}}_i$,$\pi_b$)}
\ENDFOR

\STATE \texttt{sort}$(\{\widehat{V}_i | i \in [1,B]\})$ // Sort ascending
\STATE $l \gets\lfloor \delta B \rfloor$
\STATE \textbf{Return} $\widehat{V}_l$

\end{algorithmic}
}
\caption{\textbf{Bootstrap Confidence Interval} \newline Input is an evaluation policy $\pi_{e}$, a data set of trajectories, $\mathcal{D}$, a confidence level, $\delta \in [0,1]$, and the required number of bootstrap estimates, $B$.}
\label{alg:bootstrap}
\end{algorithm}}
\end{minipage}

\end{figure}

We define \textbf{Off-PolicyEstimate} to be any method that takes a data set of trajectories, $\mathcal{D}$, the policy that generated $\mathcal{D}$, $\pi_b$, and a policy, $\pi_e$, and returns a policy value estimate, $\widehat{V}(\pi_e)$, (i.e., an off-policy estimator).
Algorithm \ref{alg:bootstrap} is a Bootstrap Confidence Interval procedure in which $\widehat{V}(\pi_e)$, as computed by \textbf{Off-PolicyEstimate}, is the statistic of interest ($\hat{\theta}$ in Section \ref{sec:bootstrap}).
We give pseudocode for a bootstrap lower bound.
The method is equally applicable to upper bounds and two-sided intervals.

The bootstrap method we present is the percentile bootstrap for confidence levels \cite{carpenter2000bootstrap}.
A more sophisticated bootstrap approach is Bias Corrected and Accelerated bootstrapping (BCa) which adjusts for the skew of the distribution of $\widehat{V}_i$.
When using \textsc{is} as \textbf{Off-PolicyEstimate}, BCa can correct for the heavy upper tailed distribution of \textsc{is} returns \cite{thomas2015policy}.\footnote{
If the \textsc{wis} estimator is used for \textbf{Off-PolicyEstimate} then Algorithm \ref{alg:bootstrap} describes a simplified version of the bootstrapping method presented by Thomas et al.\ \cite{thomas2015policy}.}

\subsection{Direct Model-Based Bootstrapping}

We now introduce our first algorithmic contribution---model-based bootstrapping (\textsc{mb-bootstrap}).
The model-based off-policy estimator, \textsc{mb}, computes $\widehat{V}(\pi_e)$ by first using all trajectories in $\mathcal{D}$ to build a model $\widehat{\mathcal{M}} = (\mathcal{S},\mathcal{A},\widehat{P},r,\gamma,\hat{d}_0)$ where $\widehat{P}$ and $\hat{d}_0$ are estimated from trajectories sampled \iid from $\pi_b$.\footnote{The reward function, $r$, may be approximated as $\hat{r}$. Our theoretical results assume $r$ is known but our proposed methods are applicable when this assumption fails to hold.}
Then \textsc{mb} estimates $\widehat{V}(\pi_e)$ as the average return of trajectories simulated in $\widehat{\mathcal{M}}$ while following $\pi_e$.
Algorithm \ref{alg:bootstrap} with the off-policy estimator \textsc{mb} as \textbf{Off-PolicyEstimate} defines \textsc{mb-bootstrap}.

If a model can capture the true MDP's dynamics or generalize well to unseen parts of the state-action space then \textsc{mb} estimates can have much lower variance than \textsc{is} estimates.
Thus we expect less variance in our $\widehat{V}_i$ estimates in Algorithm \ref{alg:bootstrap}.
However, models reduce variance at the cost of adding bias to the estimate.
Bias in the \textsc{mb} estimate of $V(\pi_e)$  arises from two sources:
\begin{enumerate}
\item When we lack data for a particular $(s,a)$ pair, we must make assumptions about how to estimate $P(\cdot|s,a)$.
\item If we use function approximation, we assume the model class from which we select $\widehat{P}$ includes the true transition model, $P$. When $P$ is outside the chosen model class then \textsc{mb}($\mathcal{D}$) can be biased because \textsc{mb}$(\mathcal{D}) - V(\pi_e) \rightarrow b$ for some constant $b\neq 0$ as $n \rightarrow \infty$.
\end{enumerate}

The first source of bias is dependent on what modeling assumptions are made.
Using assumptions which lead to more conservative estimates of $V(\pi_e)$ will, in practice, prevent \textsc{mb-bootstrap} from overestimating the lower bound. 
The second source of bias is more problematic since even as $n \rightarrow \infty$ the bootstrap model estimates will converge to a different value from $V(\pi_e)$.
In the next section we propose bootstrapping with the recently proposed weighted doubly-robust estimator in order to obtain data-efficient lower bounds in settings where model bias may be large.
Later we will present a new theoretical upper bound on model bias when $\widehat{P}$ is learned from a dataset of i.i.d.~trajectories.
This bound characterizes MDPs that are likely to produce high bias estimates.

\subsection{Weighted Doubly Robust Bootstrapping}

We also propose \textit{weighted doubly robust bootstrapping} (\textsc{wdr}-bootstrap) which combines bootstrapping with the recently proposed \textsc{wdr} off-policy estimator for settings where the \textsc{mb} estimator may exhibit high bias. 
The \textsc{wdr} estimator is based on per-decision weighted importance sampling (\textsc{pdwis}) but uses a model to reduce variance in the estimate.
The \textit{doubly robust} (\textsc{dr}) estimator has its origins in bandit problems \cite{dudik2011doubly} but was extended by Jiang and Li \shortcite{jiang2015doubly} to finite horizon MDPs.
Thomas and Brunskill \shortcite{thomas2016data-efficient} then extended \textsc{dr} to infinite horizon MDPs and combined it with weighted \textsc{is} weights to produce the weighted DR estimator.
Given a model and its state and state-action value functions for $\pi_e$, $\hat{v}_{\pi_e}$ and $\hat{q}_{\pi_e}$, the \textsc{wdr} estimator is defined as:
\begin{align*}
\scriptstyle
\text{\textsc{wdr}}(\mathcal{D}):= \text{\textsc{pdwis}}(\mathcal{D}) - \underbrace{\sum_{i=1}^n\sum_{t=0}^{L-1} \gamma^t (w_t^i\hat{q}_{\pi_e}(S_t^i,A_t^i) - w^i_{t-1}\hat{v}_{\pi_e}(S_t^i))}_\text{Control Variate Term}
\end{align*}
\normalsize
where $w_{-1} :=1$.
For \textsc{wdr}, the model value-functions are used as a control variate on the higher variance \textsc{pdwis} expectation.
The control variate term has expectation zero and thus \textsc{wdr} is an unbiased estimator of \textsc{pdwis} which is a statistically consistent estimator of $V(\pi_e)$.
Intuitively, \textsc{wdr} uses information from error in estimating the expected return under the model to lower the variance of the \textsc{pdwis} return.
We refer the reader to Thomas and Brunskill \cite{thomas2016data-efficient} and Jiang and Li \cite{jiang2015doubly} for an in-depth discussion of the \textsc{wdr} and \textsc{dr} estimators.
Since \textsc{wdr} estimates of $V(\pi_e)$ have been shown to achieve lower \textit{mean squared error} (MSE) than those of DR in several domains, we propose \textsc{wdr-bootstrap} which uses \textsc{wdr} as \textbf{Off-PolicyEstimate} in Algorithm \ref{alg:bootstrap}.

Although \textsc{wdr} is biased (since \textsc{pdwis} is biased), the statistical consistency property of \textsc{pdwis} ensures that the bootstrap estimates of \textsc{wdr-bootstrap} will converge to the correct estimate as $n$ increases.
Thus it is free of out-of-class model bias as $n \rightarrow \infty$.
Empirical results have shown that \textsc{wdr} can acheive lower MSE than \textsc{mb} in domains where the model converges to an incorrect model \cite{thomas2016data-efficient}.
However, Thomas and Brunskill also demonstrated situations where the \textsc{mb} evaluation is more efficient at acheiving low MSE than \textsc{wdr} when the variance of the \textsc{pdwis} weights is high.
We empirically analyze the trade-offs when using these estimators with bootstrapping for off-policy confidence bounds.

Note that \textsc{wdr-bootstrap} has three options for the model used to estimate the control variate: the model can  be provided (for instance a domain simulator), the model can be estimated with all of $\mathcal{D}$ and this model be used with \textsc{wdr} to compute each $\widehat{V}_i$, or we can build a new model for every bootstrap data set, $\mathcal{D}_i$, and use it to compute \textsc{wdr} for $\mathcal{D}_i$. 
In practice, an a priori model may be unavailable and it may be computationally expensive to build a model and find the value function for that model for each bootstrap data set. Thus, in our experiments, we estimate a single model with all trajectories in $\mathcal{D}$. We use the value functions of this single model to compute the \textsc{wdr} estimate for each $\mathcal{D}_i$.

\section{Trajectory Based Model Bias} \label{sec:bound}

We now present a theoretical upper bound on bias in the model-based estimate of $V(\pi_e)$.
Theorem 1 bounds the error of $\widehat{V}(\pi_e)$ produced by a model, $\widehat{M}$, as a function of the accuracy of $\widehat{M}$.
This bound provides insight into the settings in which \textsc{mb-bootstrap} is likely to be unsuccessful.
The bound is related to other model bias bounds in the literature and we discuss these in our survey of related work.
We defer the full derivation to Appendix A.
For this section we introduce the additional assumption that $L$ is finite.
All methods proposed in this paper are applicable to both finite and infinite horizon problems, however the bias upper bound is currently limited to the episodic finite horizon setting.

\begin{theorem}
For any policies, $\pi_e$ and $\pi_b$, let $p_{\pi_e}$ and $p_{\pi_b}$ be the distributions of trajectories induced by each policy. Then for an approximate model, $\widehat{M}$, with transition probabilities estimated from i.i.d.~trajectories $H \sim \pi_b$, the bias of $\widehat{V}(\pi_e)$ is upper bounded by:
\begin{multline*}
\left \vert \widehat{V}(\pi_e) - V(\pi_e) \right \vert \leq 2 L \cdot r_\mathtt{max}\sqrt{2 \mathbf{E}_{H \sim \pi_b} \left [ \rho_L^H \log \frac{p_{\pi_e}(H)}{\hat{p}_{\pi_e}(H)} \right ] }
\end{multline*}
where $\rho_L^H$ is the importance weight of trajectory $H$ at step $L$ and $\hat{p}_{\pi_e}$ is the distribution of trajectories induced by $\pi_e$ in $\widehat{M}$.
\end{theorem}

The expectation is the importance-sampled \textit{Kullback-Leibler} (KL) divergence.
The KL-divergence is an information theoretic measure that is frequently used as a similarity measure between probability distributions.
This result tells us that the bias of \textsc{mb} depends on how different the distribution of trajectories under the model is from the distribution of trajectories seen when executing $\pi$ in the true MDP. 
Since most model building techniques (e.g., supervised learning algorithms, tabular methods) build the model from $(s_t,a_t,s_{t+1})$ transitions even if the transitions come from sampled trajectories (i.e., non i.i.d.~transitions), we express Theorem 1 in terms of transitions:

\begin{corollary}
For any policies $\pi_e$ and $\pi_b$ and an approximate model, $\widehat{M}$, with transition probabilities, $\widehat{P}$, estimated with trajectories $H \sim \pi_b$, the bias of the approximate model's estimate of $V(\pi_e)$, $\widehat{V}(\pi_e)$, is upper bounded by:

{
\small
\begin{align*}
|\widehat{V}(\pi_e) - V(\pi_e)| \leq  2\sqrt{2} L \cdot  r_{max}   \sqrt{\epsilon_0 + \sum_{t=1}^{L-1} \mathbf{E}_{S_t,A_t \sim d_{\pi_b}^t}[\rho_t^H \epsilon(S_t,A_t)] }
\end{align*}
}

\normalsize
where $d_{\pi_b}^t$ is the distribution of states and actions observed at time $t$ when executing $\pi_b$ in the true MDP, $\epsilon_0:= D_{KL}(d_0||\hat{d}_0)$, and $\epsilon(s,a) = D_{KL}(P(\cdot|s,a)||\widehat{P}(\cdot|s,a)))$.
\end{corollary}


Since $P$ is unknown it is impossible to estimate the $D_{KL}$ terms in Corollary 1.
However, $D_{KL}$ can be approximated with two common supervised learning loss functions: negative log likelihood and cross-entropy.
We can express Corollary 1 in terms of either negative log-likelihood (a regression loss function for continuous MDPs) or cross-entropy (a classification loss function for discrete MDPs) and minimize the bound with observed $(s_t,a_t,s_{t+1})$ transitions.
In the case of discrete state-spaces this approximation upper bounds $D_{KL}$.
In continuous state-spaces the approximation is correct within the average differential entropy of $P$ which is a problem-specific constant.
Both Theorem 1 and Corollary 1 can be extended to finite sample bounds using Hoeffding's inequality (see Appendix A).

Corollary 1 allows us to compute the upper bound proposed in Theorem 1.
However in practice the dependence on the maximum reward makes the bound too loose to subtract off from the lower bound found by \textsc{mb-bootstrap}.
Instead, we observe it characterizes settings where the \textsc{mb} estimator may exhibit high bias.
Specifically, a \textsc{mb} estimate of $V(\pi_e)$ will have low bias when we build a model which obtains low training error under the negative log-likelihood or cross-entropy loss functions where the error due to each $(s_t,a_t,s_{t+1})$ is importance-sampled to correct for the difference in distribution.
This result holds regardless of whether or not the true transition dynamics are representable by the model class.

\section{Empirical Results}

We now evaluate \textsc{mb-bootstrap} and \textsc{wdr-bootstrap} across two policy evaluation domains.


\subsection{Experimental Domains}



The first domain is the standard MountainCar task from the RL literature \cite{SuttonBarto98}.
In this domain an agent attempts to drive an underpowered car up a hill.
The car cannot drive straight up the hill and a successful policy must first move in reverse up another hill in order to gain momentum to reach its goal.
States are discretized horizontal position and velocity and the agent may choose to accelerate left, right, or neither.
At each time-step the reward is $-1$ except for in a terminal state when it is $0$.
We build models as done by Jiang and Li \cite{jiang2015doubly} where a lack of data for a $(s,a)$ pair causes a deterministic transition to $s$.
Also, as in previous work on importance sampling, we shorten the horizon of the problem by holding action $a_t$ constant for $4$ updates of the environment state \cite{jiang2015doubly,thomas2015safe}.
This modification changes the problem horizon to $L=100$ and is done to reduce the variance of importance-sampling.
Policy $\pi_b$ chooses actions uniformly at random and $\pi_e$ is a sub-optimal policy that solves the task in approximately $35$ steps.
In this domain we build tabular models which cannot generalize from observed $(s,a)$ pairs.
We compute the model action value function, $\hat{q}_{\pi_e}$, and state value function, $\hat{v}_{\pi_e}$ with value-iteration for \textsc{wdr}.
We use Monte Carlo rollouts to estimate $\widehat{V}$ with \textsc{mb}.

\begin{figure}
\caption{CliffWorld domain in which an agent (A) must move between or around cliffs to reach a goal (G).}\label{fig:cliffworld}
\includegraphics[scale=0.65,clip=true,trim=0mm 12mm 0mm 1mm]{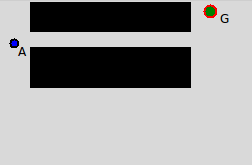}
\end{figure}

Our second domain is a continuous two-dimensional CliffWorld (depicted in Figure \ref{fig:cliffworld}) where a point mass agent navigates a series of cliffs to reach a goal, $g$.
An agent's state is a four dimensional vector of horizontal and vertical position and velocity.
Actions are acceleration values in the horizontal and vertical directions.
The reward is negative and proportional to the agent's distance to the goal and magnitude of the actions taken, $r(S_t,A_t) = ||S_t-g||_1 + ||A_t||_1$.
If the agent falls off a cliff it receives a large negative penalty.
In this domain, we hand code a deterministic policy, $\pi_d$.
Then the agent samples $\pi_e(\cdot|s)$ by sampling from $\mathcal{N}(a|\pi_d(s),\Sigma)$.
The behavior policy is the same except $\Sigma$ has greater variance.
Domain dynamics are linear with additive Gaussian noise.
We build models in two ways: linear regression (converges to true model as $n \rightarrow \infty$) and regression over nonlinear polynomial basis functions.\footnote{For each state feature, $x$, we include features $1$, $x^2$, $x^3$ but not $x$.}
The first model class choice represents the ideal case and the second is the case when the true dynamics are outside the learnable model class.
Our results refer to \textsc{mb-bootstrap}\textsuperscript{LR} and \textsc{mb-bootstrap}\textsuperscript{PR} as the \textsc{mb} estimator using linear regression and polynomial regression respectively.
These dynamics mean that the bootstrap models of MB-Bootstrap\textsuperscript{LR} and WDR-Bootstrap\textsuperscript{LR} will quickly converge to a correct model as the amount of data increases since they build models with linear regression.
On the other hand, these dynamics mean that the models of MB-Bootstrap\textsuperscript{PR} and WDR-Bootstrap\textsuperscript{PR} will quickly converge to an incorrect model since they use regression over nonlinear polynomial basis functions.
Similarly, we evaluate \textsc{wdr-bootstrap}\textsuperscript{LR} and \textsc{wdr-bootstrap}\textsuperscript{PR}.

In each domain, we estimate a 95\% confidence lower bound ($\delta=0.05$) with our proposed methods and the importance sampling BCa-bootstrap methods from Thomas et. al. \shortcite{thomas2015policy}.
To the best of our knowledge, these \textsc{is} methods are the current state-of-the-art for approximate high confidence off-policy evaluation.
We use $B=2000$ bootstrap estimates, $\widehat{V}_i$ and compute the true value of $V(\pi_e)$ with 1,000,000 Monte Carlo roll-outs of $\pi_e$ in each domain.

For each domain we computed the lower bound for $n$ trajectories where $n$ varied logarithmically.
For each $n$ we generate a set of $n$ trajectories $m$ times and compute the lower bound with each method (e.g., MB, WDR, IS) on that set of trajectories.
For Mountain Car $m=400$ and for CliffWorld $m=100$.
The large number of trials is required for the empirical error rate calculations.
When plotting the average lower bound across methods, we only average valid lower bounds (i.e., $\widehat{V}_{\delta}(\pi_e) \leq V(\pi_e)$) because invalid lower bounds raise the average which can make a method appear to produce a tighter average lower bound when it fact it has a higher error rate.

As in prior work  \cite{thomas2015off-policy}, we normalize returns for \textsc{is} and rewards for \textsc{pdis} to be between $[0,1]$.
Normalizing reduces the variance of the \textsc{is} estimator.
More importantly, it improves safety. Since the majority of the importance weights are close to zero, when the minimum return is zero, \textsc{is} tends to underestimate policy value.
Preliminary experiments showed without normalization, bootstrapping with \textsc{is} resulted in over confident bounds.
Thus, we normalize in all experiments.

\subsection{Empirical Results}

Figure \ref{fig:bound} displays the average empirical 95\% confidence lower bound found by each method in each domain.
The ideal result is a lower bound, $V_{\delta}(\pi_e)$, that is as large as possible subject to $V_{\delta}(\pi_e) < V(\pi_e)$.
Given that any statistically consistent method will achieve the ideal result as $n\rightarrow\infty$, our main point of comparison is which method gets closest the fastest.
As a general trend we note that our proposed methods---\textsc{mb-bootstrap} and \textsc{wdr-bootstrap}---get closer to this ideal result with less data than all other methods. 
Figure \ref{fig:error} displays the empirical error rate for \textsc{mb-bootstrap} and \textsc{wdr-bootstrap} and shows that they approximate the allowable 5\% error in each domain.

\begin{sloppypar}
In MountainCar (Figure \ref{fig:mc_bound}), both of our methods (\textsc{wdr-bootstrap} and \textsc{mb-bootstrap}) outperform purely \textsc{is} methods in reaching the ideal result.
We also note that both methods produce approximately the same average lower bound.
The modelling assumption that lack of data for some $(s,a)$ results in a transition to $s$ is a form of negative model bias which lowers the performance of \textsc{mb-bootstrap}.
Therefore, even though \textsc{mb} will eventually converge to $V(\pi_e)$ it does so no faster than \textsc{wdr} which can produce good estimates even when the model is inaccurate.
This negative bias also leads to \textsc{pdwis} producing a tighter bound for small data sets although it is overtaken by \textsc{mb-bootstrap} and \textsc{wdr-bootstrap} as the amount of data increases.
\end{sloppypar}

Figure \ref{fig:mc_error} shows that the \textsc{mb-bootstrap} and \textsc{wdr-bootstrap} error rate is much lower than the required error rate yet Figure \ref{fig:mc_bound} shows the lower bound is no looser.
Since \textsc{mb-bootstrap} and \textsc{wdr-bootstrap} are low variance estimators, the average bound can be tight with a low error rate.
It is also notable that since bootstrapping only approximates the 5\% allowable error rate all methods can do worse than 5\% when data is extremely sparse (only two trajectories).


\begin{figure}[h!]
\begin{subfigure}{\columnwidth} \centering
\includegraphics[scale=0.25]{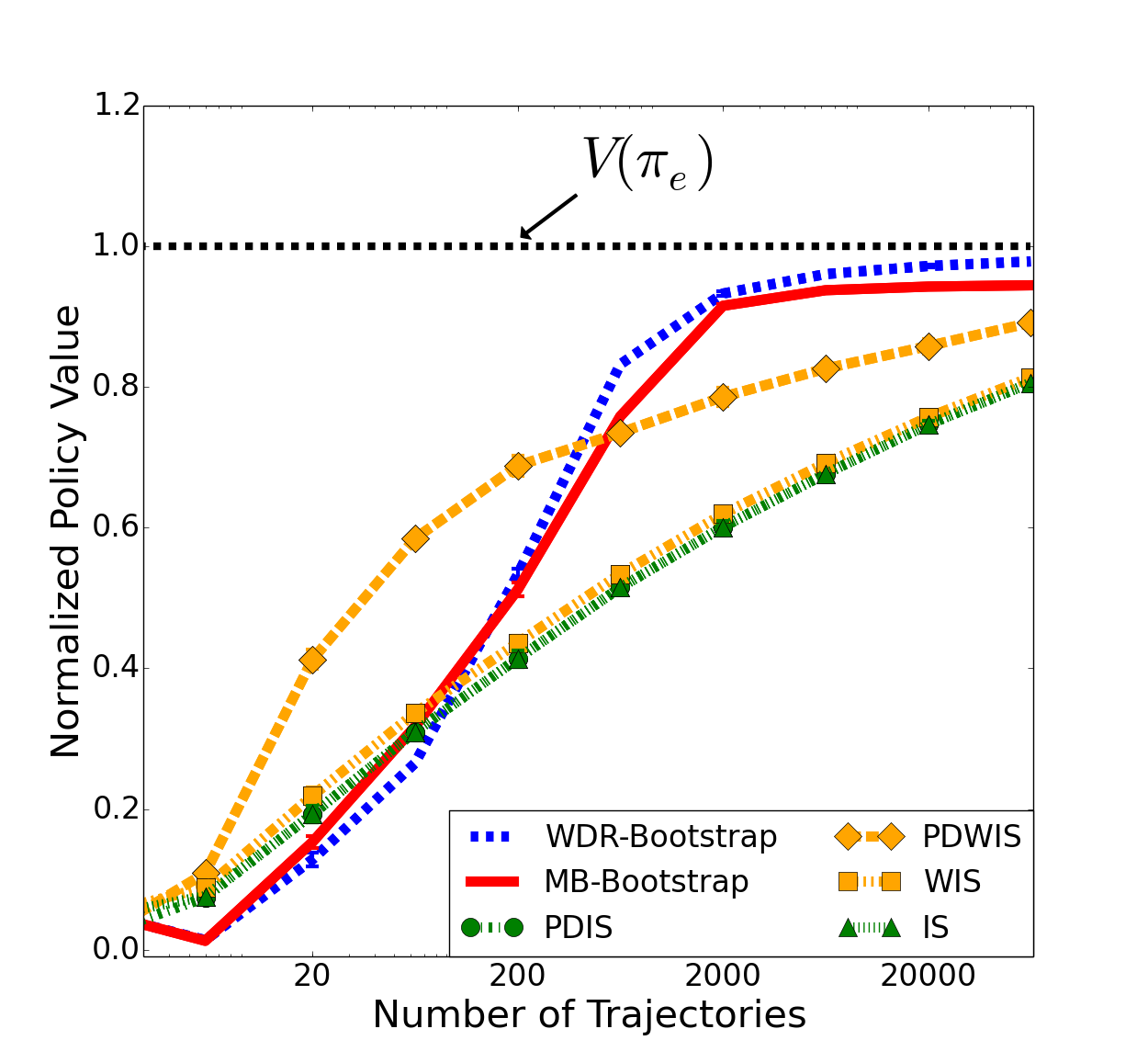}
\caption{\small{Mountain Car}}
\label{fig:mc_bound}
\end{subfigure}
\begin{subfigure}{\columnwidth} \centering
\includegraphics[scale=0.25]{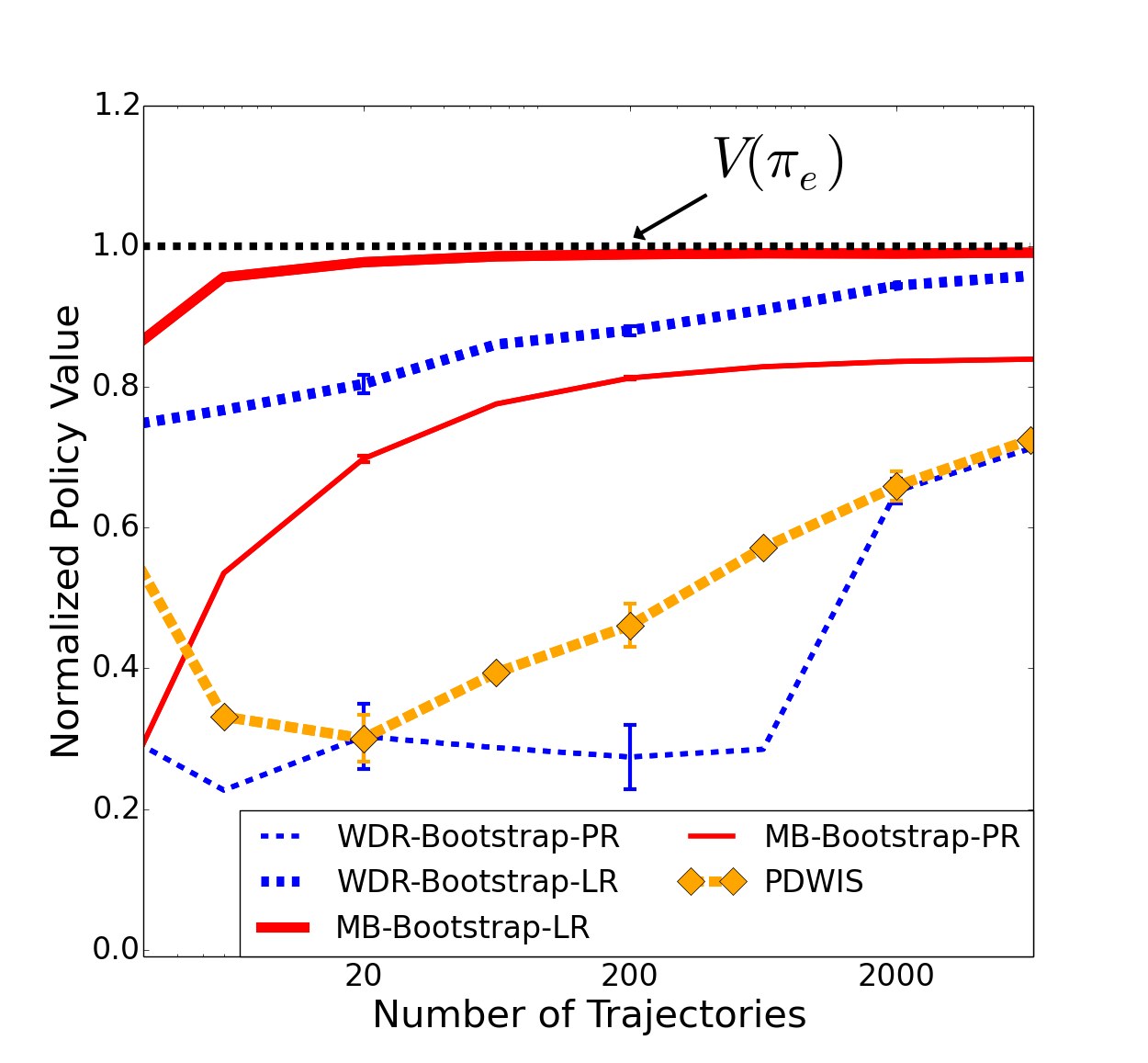}
\caption{\small{CliffWorld}}
\label{fig:cw_bound}
\end{subfigure}

\caption{\small{The average empirical lower bound for the Mountain Car and CliffWorld domains. Each plot displays the 95\% lower bound on $V(\pi_e)$ computed by each method with varying amounts of trajectories. The ideal lower bound is just below the line labelled $V(\pi_e)$. Results demonstrate that the proposed model-based bootstrapping (\textsc{mb-bootstrap}) and weighted doubly robust bootstrapping (\textsc{wdr-bootstrap}) find a tighter lower bound with less data than previous importance sampling bootstrapping methods. For clarity, we omit \textsc{is}, \textsc{wis} and \textsc{pdis} in CliffWorld as they were outperformed by \textsc{pdwis}. Error bars are for a 95\% two-sided confidence interval.}}
\label{fig:bound}
\end{figure}


In CliffWorld (Figure \ref{fig:cw_bound}), we first note that \textsc{mb-bootstrap}\textsuperscript{PR} quickly converges to a suboptimal lower bound.
In practice an incorrect model may lead to a bound that is too high (positive bias) or too loose (negative bias).
Here, \textsc{mb-bootstrap}\textsuperscript{PR} exhibits negative bias and we converge to a bound that is too loose.
More dangerous is positive bias which will make the method unsafe.
Our theoretical results suggest \textsc{mb} bias is high when evaluating $\pi_e$ since the polynomial basis function models have high training error when errors are importance-sampled to correct for the off-policy model estimation.
If we compute the bound in Section \ref{sec:bound} and subtract the value off from the bound estimated by \textsc{mb-bootstrap}\textsuperscript{PR} then the lower bound estimate will be unaffected by bias.
Unfortunately, our theoretical bound (and other model-bias bounds in earlier work) depends on the largest possible return, $L \cdot r_\mathtt{max}$ and thus removing bias in this straightforward way reduces data-efficiency gains when bias may in fact be much lower.


The second notable trend is that \textsc{wdr} is also negatively impacted by the incorrect model.
In Figure \ref{fig:cw_bound} we see that \textsc{wdr-bootstrap}\textsuperscript{LR} (correct model) starts at a tight bound and increases from there. 
\textsc{wdr-bootstrap}\textsuperscript{PR} with an incorrect model performs worse than \textsc{pdwis} until larger $n$. 
Using an incorrect model with \textsc{wdr} decreases the variance of the \textsc{pdwis} term less than the correct model would but we still expect less variance and a tighter lower bound than \textsc{pdwis} by itself.
One possibility is that error in the estimate of the model value functions coupled with the inaccurate model increases the variance of \textsc{wdr}.
This result motivates investigating the effect of inaccurate model state-value and state-action-value functions on the \textsc{wdr} control variate as these functions are certain to have error in any continuous setting.

\begin{figure}[h!]

\begin{subfigure}{\columnwidth} \centering
\includegraphics[scale=0.25]{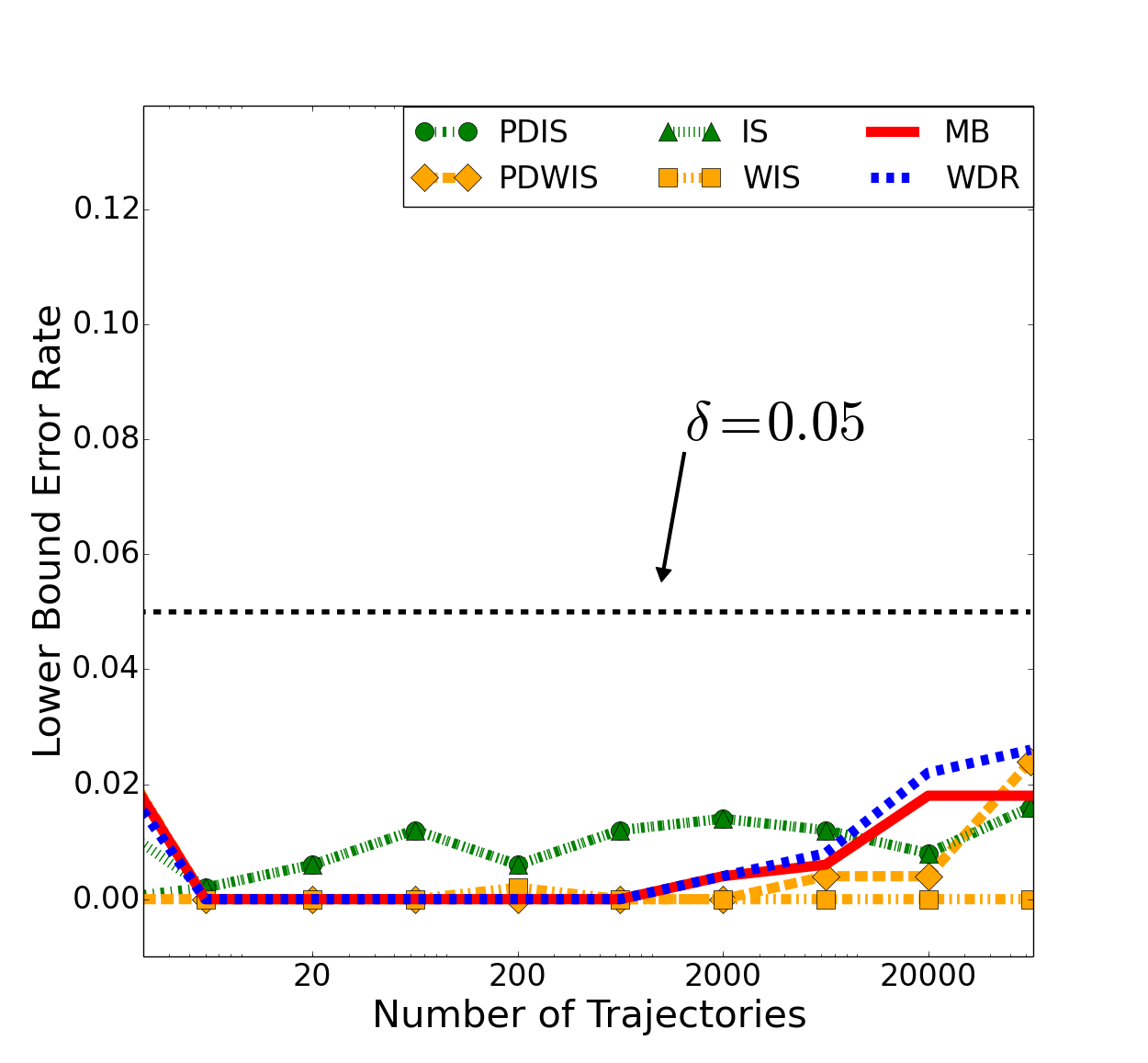}
\caption{\small{Mountain Car}}
\label{fig:mc_error}
\end{subfigure}
\begin{subfigure}{\columnwidth} \centering
\includegraphics[scale=0.25]{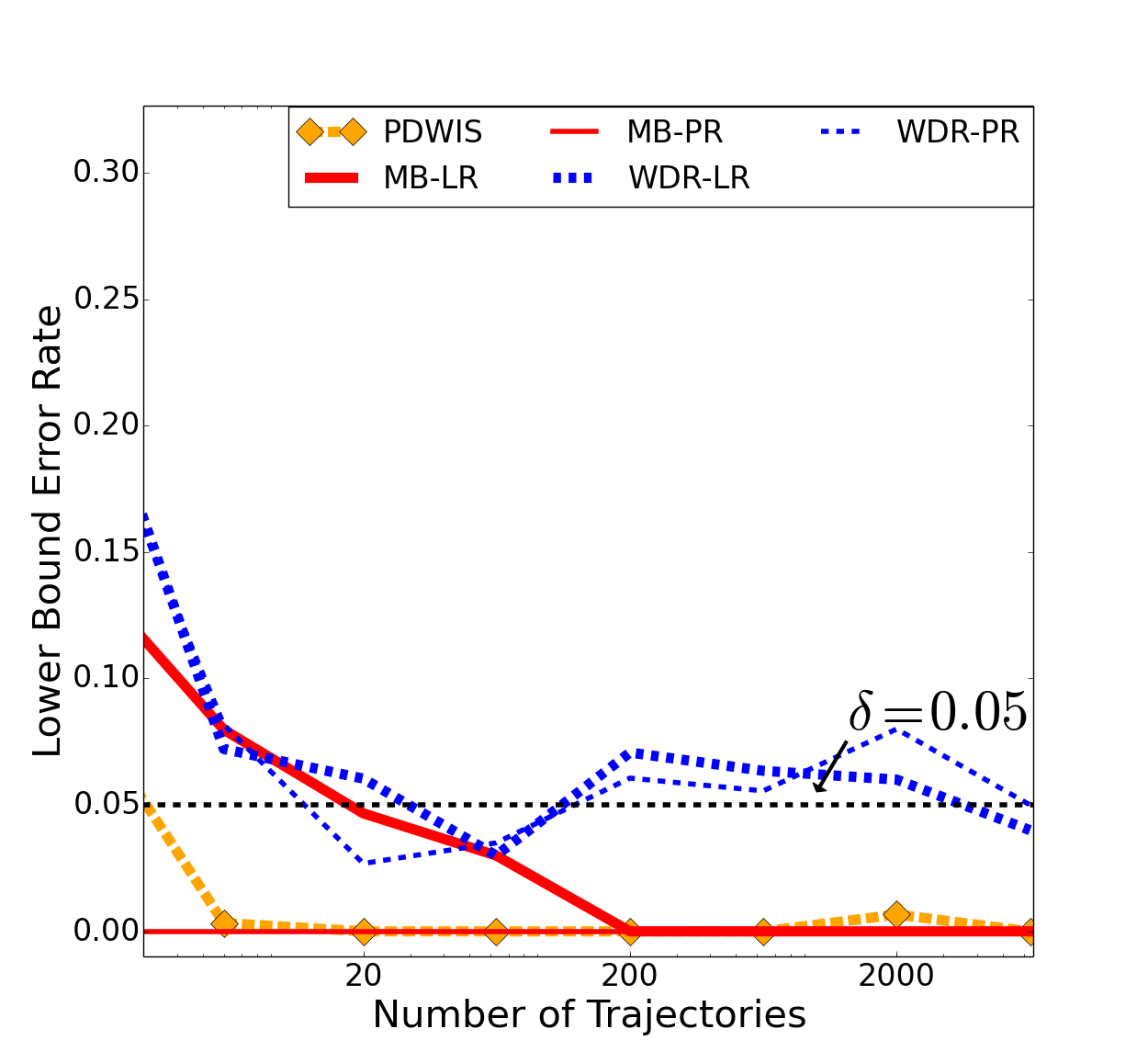}
\caption{\small{CliffWorld}}
\label{fig:cw_error}
\end{subfigure}
\caption{\small{Empirical error rate for the Mountain Car and CliffWorld domains. The lower bound is computed $m$ times for each method ($m=400$ for Mountain Car, $m=100$ for CliffWorld) and we count how many times the lower bound is above the true $V(\pi_e)$. All methods correctly approximate the allowable $5\%$ error rate for a 95\% confidence lower bound.}}
\label{fig:error}
\end{figure}

\section{Related Work}

Concentration inequalities have been used with \textsc{is} returns for lower bounds on off-policy estimates \cite{thomas2015off-policy}.
The concentration inequality approach is notable in that it produces a true probabilistic bound on the policy performance.
A similar but approximate method was proposed by Bottou et al \cite{bottou2013counterfactual}.
Unfortunately, these approachs requires prohibitive amounts of data and were shown to be far less data-efficient than bootstrapping with \textsc{is} \cite{thomas2015policy,thomas2015safe}.
Jiang and Li evaluated the DR estimator for safe-policy improvement \cite{jiang2015doubly}.
They compute confidence intervals with a method similar to the Student's $t$-Test confidence interval shown to be less data-efficient than bootstrapping \cite{thomas2015policy}.

Chow et al.\ \shortcite{chow2015robust} use ideas from robust optimization to derive a lower bound on $V(\pi_e)$ by first bounding model bias caused by error in a discrete model's transition function.
This bound is computable only if the error in each transition can be bounded and is inapplicable for estimating bias in continuous state-spaces.
Model-based PAC-MDP methods can be used to synthesize policies which are approximately optimal with high probability \cite{fu2014probably}.
These methods are only applicable to discrete MDPs and require large amounts of data.

Other bounds on the error in estimates of $V(\pi)$ with an inaccurate model have been introduced for discrete MDPs \cite{kearns2002near,strehl2009reinforcement}.
In contrast, we present a bound on model bias that is computable in both continuous and discrete MDPs.
Ross and Bagnell introduce a bound similar to Corollary 2 for model-based policy improvement but assume that the model is estimated from transitions sampled i.i.d.~from a given exploration distribution \cite{ross2012agnostic}.
Since we bootstrap over trajectories their bound is inapplicable to our setting.
Paduraru introduced tight model bias bounds for i.i.d.~sampled transitions from general MDPs and i.i.d.~trajectories from directed acyclic graph MDPs \cite{paduraru2013off}.
We made no assumptions on the structure of the MDP when deriving our bound.

Other previous work has used bootstrapping to handle uncertainty in RL.
The \textsc{texplore} algorithm learns multiple decision tree models from subsets of experience to represent uncertainty in model predictions \cite{hester2010real}.
White and White \cite{white2010interval} use time-series bootstrapping to place confidence intervals on value-function estimation during policy learning.
Thomas and Brunskill introduce an estimate of the model-based estimator's bias using a combination of \textsc{wdr} and bootstrapping \cite{thomas2016data-efficient}.
While these methods are related through the combination of bootstrapping and RL, none address the problem of confidence intervals for off-policy evaluation.

\section{Discussion}

We have proposed two bootstrapping methods that incorporate models to produce tight lower bounds on off-policy estimates.
We now describe their advantages and disadvantages and make recommendations about their use in practice.

\paragraph{Model-based Bootstrapping}

Clearly, \textsc{mb-bootstrap} is influenced by the quality of the estimated transition dynamics.
If \textsc{mb-bootstrap} can build models with low importance-sampled approximation error then we can expect it to be more data-efficient than other methods.
This data-efficiency comes at a cost of potential bias. 
Our theoretical results show that bias is unavoidable for some model-class choices.
However if the chosen model-class can be learned with low approximation error on $\mathcal{D}$ then model bias will be low.
In practice model prediction error for off-policy evaluation may be evaluated with a held out subset of $\mathcal{D}$ (i.e., model validation error).
If the model fails to generalize to unseen data then another off-policy method is preferable.
Importance-sampling the test error gives a measure of how well a model estimated with trajectories from $\pi_b$ will generalize for evaluating $\pi_e$.


\paragraph{Weighted Doubly-Robust Bootstrap}
Our proposed \textsc{wdr-bootstrap} method provides a low variance and low bias method of high confidence off-policy evaluation.
These two properties allow \textsc{wdr-bootstrap} to outperform different variants of \textsc{is} and sometimes perform as well or better than \textsc{mb-bootstrap}.
In contrast to \textsc{mb-bootstrap}, \textsc{wdr-bootstrap} achieves data-efficient lower bounds while remaining free of model bias.
Since \textsc{wdr-bootstrap} is free of model bias, it should be the preferred method if model quality is unknown or the domain is hard to model.

A disadvantage of \textsc{wdr-bootstrap} is that it requires the model's value functions be known for all states and state-action pairs that occur along trajectories in $\mathcal{D}$.
In continuous state and action spaces this requires either function approximation or Monte Carlo evaluation.
The variance of either method can increase the variance of the \textsc{wdr} estimate.
Note that \textsc{wdr} remains bias free provided $\hat{v}_{\pi_e}(s) = \mathbf{E}_{A \sim \pi_e(\cdot|s)}\left[\hat{q}_{\pi_e}(s,A)\right]$ which ensures the control variate term has expected value zero even if $\hat{q}_{\pi_e}$ is a biased estimate of the policy's action value function under the model.

A second limitation of \textsc{wdr} is that it biases $\widehat{V}(\pi_e)$ towards $V(\pi_b)$ when the trajectory dataset is small.
This bias is problematic for confidence bounds when $V(\pi_b) > V(\pi_e)$ as the lower bound on $V(\pi_e)$ will exhibit positive bias.
While the bias is a problem for general high confidence off-policy evaluation it is harmless in the specific case of high confidence off-policy improvement.
In this setting the purpose of the test is to decide if we are confident that the evaluation policy is better than the behavior policy.
If in fact $V(\pi_b) > V(\pi_e)$ the lower bound will still be less than $V(\pi_b)$ and an unsafe policy improvement step is avoided.
Similarly, if we know $V(\pi_b) > V(\pi_e)$ than \textsc{wdr-bootstrap} will likely have less variance than \textsc{is}-based methods and less bias than \textsc{mb-bootstrap}.


\paragraph{Importance Sampling Methods}
For general high confidence off-policy evaluation tasks in which model estimation error is high, \textsc{is}- or \textsc{pdis}-based bootstrapping provides the safest approximate high confidence off-policy evaluation method.
We have noted that normalizing returns and rewards is an important factor in using these methods safely.
Since most importance weights are close to zero the \textsc{is} estimate will be pulled towards zero which corresponds to underestimating value.
When safety is critical, underestimating is preferable to overestimating for a lower bound.
Our experiments used normalization as we found unnormalized returns have too high of variance to be used safely with bootstrapping.

Finally, in settings where safety must be strictly guaranteed, concentration inequalities with \textsc{is} have been shown to outperform other exact methods \cite{thomas2015off-policy}. If the data is available, then exact methods are preferred for their theoretical guarantees.

\paragraph{Special Cases}
Two special cases that occur in real world high confidence off-policy evaluation are deterministic policies and unknown $\pi_b$.
When $\pi_e$ is deterministic, one should use \textsc{mb-bootstrap} since the importance weights equal zero at any time step that $\pi_b$ chose action $a_t$ such that $\pi_e(A_t=a_t|S_t) = 0$.
Deterministic $\pi_b$ are problematic for any method since they produce trajectories which lack a variety of action selection data. 
We also note that importance-sampled training error for assessing model quality is inapplicable to this setting.
Unknown $\pi_b$ occur when we have domain trajectories but no knowledge of the policy that produced the trajectories.
For example, a medical domain could have data on treatments and outcomes but the doctor's treatment selection policy be unknown.
In this setting, importance sampling methods cannot be applied and \textsc{mb-bootstrap} may be the only way to provide a confidence interval on a new policy.
A current gap in the literature exists for these special cases with unbiased bounds in continuous settings.

\section{Conclusion and Future Work}

\begin{sloppypar}
We have introduced two straightforward yet novel methods---\textsc{mb-bootstrap} and \textsc{wdr-bootstrap}---that approximate confidence intervals for off-policy evaluation with bootstrapping and learned models.
Empirically, our methods yield superior data-efficiency and tighter lower bounds on the performance of the evaluation policy than state-of-the-art importance sampling based methods.
We also derived a new bound on the expected bias of \textsc{mb} when learning models that minimize error over a dataset of trajectories sampled \iid from an arbitrary policy.
Together, the empirical and theoretical results enhance our understanding of bootstrapping for off-policy confidence intervals and allow us to make recommendations on the settings where different methods are appropriate.
\end{sloppypar}

Our ongoing research agenda includes applying these techniques within robotics.
In robotics, off-policy challenges may arise from data scarcity, deterministic policies, or unknown behavior policies (e.g., demonstration data).
While these challenges suggest \textsc{mb-bootstrap} is appropriate, robots may exhibit complex, non-linear dynamics that are hard to model.
Understanding and finding solutions for high confidence off-policy evaluation across robotic tasks may inspire innovation that can be applied to other domains as well.

\section*{Acknowledegments}
We would like to thank Phil Thomas, Matthew Hausknecht, Daniel Brown, Stefano Albrecht, and Ajinkya Jain for useful discussions and insightful comments and Emma Brunskill and Yao Liu for pointing out an error in an earlier version of this work.
This work has taken place in the Personal Autonomous Robotics Lab (PeARL) and Learning Agents Research Group (LARG) at the Artificial Intelligence Laboratory, The University of Texas at Austin. PeARL research is supported in part by NSF (IIS-1638107, IIS-1617639). LARG research is supported in part by NSF (CNS-1330072, CNS-1305287, IIS-1637736, IIS-1651089), ONR (21C184-01), AFOSR (FA9550-14-1-0087), Raytheon, Toyota, AT\&T, and Lockheed Martin. Josiah Hanna is supported by an NSF Graduate Research Fellowship. Peter Stone serves on the Board of Directors of, Cogitai, Inc.  The terms of this arrangement have been reviewed and approved by the University of Texas at Austin in accordance with its policy on objectivity in research.

\begin{appendices}

\section{}

This appendix proves all theoretical results contained in the main text.
For convenience, proofs are given for discrete state and action sets.
Results hold for continuous states and actions by replacing summations over states and actions with integrals and changing probability mass functions to probability density functions.

\subsection{Model Bias when Evaluation and Behavior Policy are the Same}

\begin{lemma}
For any policy $\pi$, let $p_\pi$ be the distribution of trajectories generated by $\pi$ and $\hat{p}_\pi$ be the distribution of trajectories generated by $\pi$ in an approximate model, $\widehat{M}$. The bias of an estimate, $\widehat{V}(\pi)$, under $\widehat{M}$ is upper bounded by:
\[ \left \vert V(\pi) - \widehat{V}(\pi) \right \vert \leq 2\sqrt{2} L \cdot r_{max}\sqrt{D_{KL}(p_\pi||\hat{p}_\pi)}\]
where $D_{KL}(p_\pi||\hat{p}_\pi)$ is the Kullback-Leibler (KL) divergence between probability distributions $p_\pi$ and $\hat{p}_\pi$.
\end{lemma}

\begin{proof}

\[ \left \vert V(\pi) - \widehat{V}(\pi) \right \vert = \left |\displaystyle\sum_h p_\pi(h) g(h) - \displaystyle\sum_h \hat{p}_\pi(h) g(h)\right |\]


From Jensen's inequality and the fact that $g(h) \geq 0$:

\[  \left \vert V(\pi) - \widehat{V}(\pi) \right \vert \leq \displaystyle\sum_h \left |p_\pi(h) - \hat{p}_\pi(h) \right | g(h)\]



After replacing $g(h)$ with the maximum possible return, $g_\mathtt{max}:=L \cdot r_\mathtt{max}$, and factoring it out of the summation, we can use the definition of the total variation ($ D_{TV}(p||q) = \frac{1}{2}\displaystyle\sum_x \left |p(x) - q(x)\right |$) to obtain:

\[ \left \vert V(\pi) - \widehat{V}(\pi) \right \vert \leq 2D_{TV}(p_\pi||\hat{p}_\pi) \cdot  g_{max}\]



The definition of $g_\mathtt{max}$ and Pinsker's inequality ($D_{TV}(p||q) \leq \sqrt{2D_{KL}(p||q)}$) completes the proof.
\end{proof}

\subsection{Bounds in terms of behavior policy data}

\setcounter{theorem}{0}
\begin{theorem}
For any policies $\pi_e$ and $\pi_b$ let $p_{\pi_e}$ and $p_{\pi_b}$ be the distributions of trajectories induced by each policy. Then for an approximate model, $\widehat{M}$, estimated with \iid trajectories, $H \sim \pi_b$, the bias of the estimate of $V(\pi_e)$ with $\widehat{M}$, $\widehat{V}(\pi_e)$, is upper bounded by:

\[\left |\widehat{V}(\pi_e) - V(\pi_e)\right | \leq 2\sqrt{2} L \cdot r_{max}\sqrt{\mathbf{E}_{H \sim \pi_b}\left [\rho_L^H \log\frac{p_{\pi_e}(H)}{\hat{p}_{\pi_e}(H)} \right ]}\]

where $\rho_L^H$ is the importance weight of trajectory $H$ at step $L$ and $\hat{p}_{\pi_e}$ is the distribution of trajectories induced by $\pi_e$ in $\widehat{M}$.
\end{theorem}

\begin{proof}
Theorem 1 follows from Lemma 1 with the importance-sampling identity (i.e., importance-sampling the expectation in Lemma 1 so that it is an expectation with $H \sim \pi_b$).
The transition probabilities cancel in the importance weight, $\frac{p_{\pi_e}(H)}{p_{\pi_b}(H)}$, leaving us with $\rho_L^H$ and completing the proof.
\end{proof}

\subsection{Bounding Theorem 1 in terms of a Supervised Loss Function}

We now express Theorem 1 in terms of an expectation over transitions that occur along sampled trajectories.

\setcounter{corollary}{0}
\begin{corollary}
For any policies $\pi_e$ and $\pi_b$ and an approximate model, $\widehat{M}$, with transition probabilities, $\widehat{P}$, estimated with trajectories $H \sim \pi_b$, the bias of the approximate model's estimate of $V(\pi_e)$, $\widehat{V}(\pi_e)$, is upper bounded by:
\small
\begin{align*}
|\widehat{V}(\pi_e) - V(\pi_e)| \leq  2\sqrt{2} L \cdot  r_{max}\sqrt{\epsilon_0 + \sum_{t=1}^{L-1} \mathbf{E}_{S_t,A_t \sim d_{\pi_b}^t}[\rho_t^H \epsilon(S_t,A_t)]}
\end{align*}
\normalsize
where $d_{\pi_b}^t$ is the distribution of states and actions observed at time $t$ when executing $\pi_b$ in the true MDP, $\epsilon_0:= D_{KL}(d_0||\hat{d}_0)$, and $\epsilon(s,a) = D_{KL}(P(\cdot|s,a)||\widehat{P}(\cdot|s,a)))$.
\end{corollary}

Corollary 1 follows from Theorem 1 by equating the expectation to an expectation in terms of $(S_t,A_t,S_{t+1})$ samples:

\begin{proof}

\[ \mathbf{E}_{H \sim \pi_b} \left [\rho_L^H \log \frac{p_{\pi_e}(H)}{\hat{p}_{\pi_e}(H)} \right ] = \displaystyle\sum_h p_{\pi}(h) \log \frac{p_{\pi}(h)}{\hat{p}_{\hat{\pi}}(h)}\]

\begin{multline*}
= \displaystyle\sum_{s_0} \displaystyle\sum_{a_0} \cdot \cdot \cdot \displaystyle\sum_{s_{L-1}} \displaystyle\sum_{a_{L-1}} d_0(s_0)\pi_b(a_0|s_0) \cdot \cdot \cdot P(s_{L-1}|s_{L-2},a_{L-2}) \cdot \\ \pi_b(a_{L-1}|s_{L-1})  \rho_L^H \log \frac{p(s_0) \cdot \cdot \cdot P(s_{L-1}|s_{L-2},a_{L-2}) }{\hat{p}(s_0) \cdot \cdot \cdot \widehat{P}(s_{L-1}|s_{L-2},a_{L-2}) }
\end{multline*}

Using the logarithm property that $\log(ab) = \log(a) + \log(b)$ and rearranging the summation allows us to marginalize the probabilities that do not appear in the logarithm.

\begin{gather*}
 = \displaystyle\sum_{s_0} d_0(s_0) \log \frac{d_0(s_0)}{\hat{d}_0(s_0)} \\ + \displaystyle\sum_{t=1}^{L-1} \displaystyle\sum_{s_0}d_0(s_0) \cdot \cdot \displaystyle\sum_{s_t} \rho_L^H P(s_t|s_{t-1},a_{t-1}) \log\frac{P(s_t|s_{t-1},a_{t-1})}{\widehat{P}(s_t|s_{t-1},a_{t-1})} 
\end{gather*}
\begin{sloppypar}
Define the probability of observing $s$ and $a$ at time $t+1$ when following $\pi_b$ recursively as 
$d_{\pi_b}^{t+1}(s,a):= \displaystyle\sum_{s_t,a_t} d_{\pi_b}^t(s_t,a_t) P(s|s_t,a_t)\pi_b(a|s)$ where $d_{\pi_b}^1(s,a):=d_0(s)\pi_b(a|s)$. Using this definition to simplify:
\end{sloppypar}
\begin{gather*}
 = D_{KL}(d_0||\hat{d}_0) + \displaystyle\sum_{t=1}^{L-1} \mathbf{E}_{S,A \sim d^t_{\pi_b}}\left [\rho_L^H D_{KL}(P(\cdot|S,A)||\widehat{P}(\cdot|S,A)) \right] 
\end{gather*}
\end{proof}

\begin{sloppypar}
We relate $D_{KL}$ to two common supervised learning loss functions so that we can minimize Corollary 1 with $(S_t,A_t,S_{t+1})$ samples.
$D_{KL}(P||\widehat{P}) = H[P,\widehat{P}] - H[P]$ where $H[P]$ and $H[P,\widehat{P}]$ are entropy and cross-entropy respectively.
For discrete distributions, $H[P,\widehat{P}] - H[P] \leq H[P,\widehat{P}]$ since entropy is always positive.
This fact allows us to upper bound $D_{KL}$ with the cross-entropy loss function.
The cross-entropy loss function is equivalent to the expected negative log likelihood loss function:
$H(P(\cdot|s,a),\widehat{P}(\cdot|s,a))) = \mathbf{E}_{S' \sim P(\cdot|s,a)}[-\log \widehat{P}(S'|s,a)] = \mathbf{E}_{S' \sim P(\cdot|s,a)}[\operatorname{nlh}(\widehat{P},s,a,S')]$ where $\operatorname{nlh}(P,s,a,s') := -\log(P(s'|s,a))$. 
Thus our bound applies to maximum likelihood model learning.
For continuous domains where the transition function is a probability density function, entropy can be negative so the negative log-likelihood or cross-entropy loss functions will not always bound model bias.
In this case, our bound approximates the true bias bound to within a constant.
\end{sloppypar}

\subsection{Finite Sample Bounds}

Theorem 1 can be expressed as a finite-sample bound by applying Hoeffding's inequality to bound the expectation in the bound.


\begin{corollary}
For any policies $\pi_e$ and $\pi_b$ and an approximate model, $\widehat{M}$, with transition probabilities, $\widehat{P}$, estimated with transitions, $(s,a)$, from trajectories $H \sim \pi_b$, and after observing $m$ trajectories then with probability $\alpha$, the bias of the approximate model's estimate of $V(\pi_e)$, $\widehat{V}(\pi_e)$, is upper bounded by:
\small{
\begin{gather*}
\left |\widehat{V}(\pi_e) - V(\pi_e) \right | \leq  2 L \cdot  r_{max} \cdot \\ \sqrt{ 2\bar{\rho}_L\sqrt{\frac{\ln(\frac{1}{\alpha})}{2m}} - \frac{1}{m} \sum_{h\in \mathcal{D}} \rho_L^h \left ( \log \hat{d}_0(s_1) + \sum_{t=1}^{L-1} \log \widehat{P}(s_{t+1}|s_t,a_t)\right )}
\end{gather*}
}\normalsize
where $\bar{\rho}_L$ is an upper bound on the importance ratio, i.e., for all $\rho_L^h$, $\rho_L^h < \bar{\rho}_L$.
\end{corollary}

\begin{proof}
Corollary 2 follows from applying Hoeffding's Inequality to Theorem 1 and then expanding $D_\mathtt{KL}(p||\hat{p})$ to be in terms of samples as done in the derivation of Corollary 1.
We then drop logarithm terms which contain the unknown $d_0$ and $P$ functions.
Dropping these terms is equivalent to expressing Corollary 2 in terms of the cross-entropy or negative log-likelihood loss functions.
\end{proof}

\end{appendices}

\bibliographystyle{plain}

\end{document}